\pdfoutput=1

\documentclass[11pt]{article}
\usepackage[utf8]{inputenc}
\usepackage[T1]{fontenc}

\usepackage{coling2020}
\usepackage{times}
\usepackage{url}
\usepackage{amsmath}
\usepackage{amssymb}
\usepackage{amsthm}
\usepackage{latexsym}
\usepackage{csquotes}
\usepackage{booktabs}
\usepackage{multirow}
\usepackage{xcolor}
\usepackage{subcaption}
\usepackage{tikz-dependency}
\usepackage{verbatimbox}
\usepackage{blindtext}

\newcommand{\sref}[1]{Section~\ref{sec:#1}}

\newcommand{\fref}[1]{Figure~\ref{fig:#1}}

\newtheorem{theorem}{Theorem}

\DeclareSymbolFont{extraup}{U}{zavm}{m}{n}
\DeclareMathSymbol{\varclub}{\mathalpha}{extraup}{84}
\DeclareMathSymbol{\vardiamond}{\mathalpha}{extraup}{87}

\usepackage{todonotes}

\colingfinalcopy 

\title{Maximum Spanning Trees Are Invariant to Temperature Scaling\\in Graph-based Dependency Parsing}

\author{Stefan Grünewald\\
	Institut für Maschinelle Sprachverarbeitung, University of Stuttgart\\
	Bosch Center for Artificial Intelligence, Renningen, Germany\\
	\texttt{stefan.gruenewald@de.bosch.com}
}

\date{\today}

\begin{document}
\maketitle
\begin{abstract}
Modern graph-based syntactic dependency parsers operate by predicting, for each token within a sentence, a probability distribution over its possible syntactic heads (i.e., all other tokens) and then extracting a maximum spanning tree from the resulting log-probabilities.
Nowadays, virtually all such parsers utilize deep neural networks and may thus be susceptible to miscalibration (in particular, overconfident predictions).
In this paper, we prove that \textit{temperature scaling}, a popular technique for post-hoc calibration of neural networks, cannot change the output of the aforementioned procedure.
We conclude that other techniques are needed to tackle miscalibration in graph-based dependency parsers in a way that improves parsing accuracy.
\end{abstract}

\section{Introduction}
\textit{Syntactic dependency parsing} refers to the task of predicting, for a given sentence, the grammatical relations between its tokens.
Most commonly, the output is a \textit{dependency tree}, i.e., a graph structure in which each token constitutes a node and is assigned exactly one parent (its syntactic head). The parent may be either one of the other words in the sentence or an additional, implicit \textsc{Root} node. 
For the dependency tree to be valid, each token must be reachable from \textsc{Root}.\footnote{In practice, there is often the additional constraint that \textsc{Root} can only have one outgoing edge. We will discuss this later in the paper.}
\fref{dep_example} shows such a structure for the sentence \enquote{Mary likes fluffy cats.}\footnote{Edges in a dependency tree may additionally be labelled according to the grammatical functions they represent. For example, because \textit{cats} is the grammatical object of \textit{likes}, the edge between them may be labelled \textit{obj}. We will ignore edge labels for the remainder of this paper, as they are not relevant to our discussion.}

\begin{figure}[h]
\centering
\begin{minipage}{.5\textwidth}
  \centering
      \begin{dependency}[theme=simple, arc edge, arc angle=80]
		\begin{deptext}
			Mary \& likes \& fluffy \& cats\\
		\end{deptext}
		\deproot[edge height=4em, label style={font=\normalsize}]{2}{\textsc{Root}}
		\depedge[hide label]{2}{1}{nsubj}
		\depedge[hide label]{2}{4}{obj}
		\depedge[hide label]{4}{3}{amod}
	  \end{dependency}
    \captionof{figure}{Example dependency tree.}
    \label{fig:dep_example}
\end{minipage}%
\begin{minipage}{.5\textwidth}
  \centering
    \footnotesize
    \bgroup
    \def\arraystretch{1.19}
    \begin{tabular}{l|ccccc}
         & \textsc{Root} & Mary & likes & fluffy & cats \\
         \midrule
      Mary   & $0.01$ & $0.02$ & $0.88$ & $0.07$ & $0.02$\\
      likes  & $0.95$ & $0.01$ & $0.00$ & $0.03$ & $0.01$\\ 
      fluffy & $0.09$ & $0.13$ & $0.05$ & $0.02$ & $0.71$\\ 
      cats   & $0.03$ & $0.10$ & $0.74$ & $0.12$ & $0.01$\\ 
    \end{tabular}
    \egroup
    \captionof{figure}{Syntactic head probabilities.}
    \label{fig:head_probs}
\end{minipage}
\end{figure}

\textit{Graph-based dependency parsing} is a technique for predicting dependency trees.
In its usual formulation, the approach entails using some machine learning classifier to predict, for each token within the input sentence, a probability distribution over its possible syntactic heads (i.e., all other tokens in the sentence, as well as \textsc{Root}), as shown in the rows of \fref{head_probs}.
In a second step, the logarithms of the resulting probabilities are then interpreted as edge weights between pairs of nodes corresponding to tokens and a maximum spanning tree is extracted, using, e.g., the Chu-Liu/Edmonds algorithm \cite{chuliu1965shortest,edmonds1967optimum}.
This maximum spanning tree is then returned as the dependency tree for the input sentence.

The \textit{calibration} of a machine learning classifier refers to its ability to generate output probabilities that are representative of the actual correctness likelihoods.
For example, in a well-calibrated dependency parser, roughly 80\,\% of edges that are predicted with a probability of $0.8$ should actually be present in the gold data. 
However, as \newcite{guo2017calibration} show, modern neural networks are often miscalibrated and prone to overconfident predictions.
In the context of graph-based dependency parsing, this would mean that the probability distributions over syntactic heads are too concentrated on the parent considered the most likely, not properly reflecting the uncertainty in the syntactic attachment of the tokens.
Since the final dependency trees are extracted using an MST algorithm that operates on the (log-)probabilities returned by the classifier (and not merely the highest-scoring parent for each token), overconfidence in the syntactic attachments of the tokens might not only lead to overconfidence in the parse itself, but also to incorrect parses. Thus, good estimates of token attachment uncertainty may not only improve the uncertainty estimates of the final trees, but also their accuracy.

One popular technique for post-hoc calibration of neural networks is \textit{temperature scaling} \cite{guo2017calibration}.
This approach works by dividing the unnormalized output values (\enquote{logits}) of the neural network by a constant $T \in \mathbb{R}_{>0}$ (the \enquote{softmax temperature}) before applying the softmax function, i.e., given a vector $\mathbf{x}$ of unnormalized scores, the vector of probabilities is computed as $\text{softmax}(\mathbf{x}/T)$ rather than simply $\text{softmax}(\mathbf{x})$. $T$ is optimized w.r.t. negative log-likelihood on a validation set and usually $T > 1$, resulting in a \enquote{smoother} probability distribution (i.e., one with higher entropy, or higher uncertainty) that corrects for the model's overconfidence.

In this paper, we show that applying temperature scaling to the logits of a graph-based dependency parser (i.e., the unnormalized values of the head probability distributions; the input values for the rows in \fref{head_probs}) does not change the resulting dependency trees. Although it is already known that temperature scaling does not change predictions for individual classifications, this is nonetheless a somewhat surprising result since the maximum spanning trees in dependency parsing depend on \textit{all} (log-)probabilities provided by the underlying classifier.
We conclude that calibration methods that simply re-scale softmax probabilities are not suited to improve the accuracy of graph-based dependency parsers.

The remainder of this paper is structured as follows.
\sref{definitions} sets up definitions for temperature scaling in graph-based dependency parsing;
\sref{main_proof} proves our main claim;
\sref{conclusion} concludes the paper.

\section{Definitions}
\label{sec:definitions}

Let $V = \{ v_1, \ldots v_n \}$ be a set of nodes.
In practice, these are the individual tokens of the sentence, as well as \textsc{Root}.

Let $X \in \mathbb{R}^{n \times n}$ a matrix of scores between nodes.\footnote{In practice, the matrix row corresponding to the \textsc{Root} token is often ommitted, i.e., $X \in \mathbb{R}^{(n-1) \times n}$ (see \fref{head_probs}). This is because \textsc{Root} will always constitute the root of the dependency tree and thus not have a parent of its own, making the weights of incoming edges irrelevant. However, this has no impact on our overall argument.}
These are the unnormalized output values (\enquote{logits}) of the dependency classifier.

Given $X$ and a \enquote{softmax temperature} $T \in \mathbb{R}_{>0}$, we define a parametrized weight function $w_{X,T}: V^2 \mapsto \mathbb{R}$ as 
\begin{align}
    w_{X,T}(v_i, v_j) &= \log \left( \frac{e^{\frac{x_{ij}}{T}}}{\sum_{k=1}^n e^{\frac{x_{kj}}{T}}} \right)\\
                  &= \log \left( e^{\frac{x_{ij}}{T}} \right) - \log \left( \sum_{k=1}^n e^{\frac{x_{kj}}{T}} \right)\\
                  &= \frac{x_{ij}}{T} - \log \left( \sum_{k=1}^n e^{\frac{x_{kj}}{T}} \right)
\end{align}
$w_{X,T}$ assigns weights to all edges between two nodes; for each node $v_j$, the weights of incoming edges are the log-softmax values of the unnormalized scores in the corresponding row of $X$. For $T = 1$, this is the standard approach in graph-based dependency parsing.

$D = \langle V, V^2 \rangle$ and $w_{X,T}$ form a complete, directed, weighted graph.
Given $D$ and a designated root node $r \in V$, an \textit{arborescence} (directed spanning tree) is a subgraph $A = (r, V, E)$ with $E \subseteq V^2$ such that (a) each non-root node has exactly one incoming edge, and (b) $A$ has no cycles.

Without loss of generality, assume that $r = v_1$. An arborescence $A$ then induces a predecessor function $\pi: \{2, \ldots, n\} \mapsto \{1, \ldots, n\}$ that maps the index of each non-root node to the index of its parent. 
We can now define the weight of $A$ as the sum of its edges:
\begin{equation}
    w_{X,T}(A) = \sum_{i=2}^n w_{X,T}(v_{\pi(i)}, v_i)
\end{equation}
and we call $A$ a maximum arborescence of $(D, r)$ and $w_{X,T}$ if $w_{X,T}(A) \geq w_{X,T}(A')$ for all arborescences $A'$ of $(D, r)$.

\section{Proof of Main Claim}
\label{sec:main_proof}

Our main claim is that a maximum arborescence $A$ of $(D, r)$ and some $w_{X,T}$ is also a maximum arborescence of $(D, r)$ and any other $w_{X,T'}$.
To prove this, we first prove the following, more general statement:
\begin{theorem}
\label{thrm:generalized_claim}
If $A$ and $A'$ are arborescences of $(D, r)$ and $w_{X,T}(A) \geq w_{X,T}(A')$ for some $T \in \mathbb{R}_{>0}$, then $w_{X,T'}(A) \geq w_{X,T'}(A')$ for all $T' \in \mathbb{R}_{>0}$.
\end{theorem}

\begin{proof}
We note the following equality:

\begin{align}
w_{X,T}(A) - w_{X,T}(A') &= \sum_{i=2}^n w_{X,T}(v_{\pi(i)}, v_i) - \sum_{i=2}^n w_{X,T}(v_{\pi'(i)}, v_i) \\
             &= \sum_{i=2}^n \left( w_{X,T}(v_{\pi(i)}, v_i) - w_{X,T}(v_{\pi'(i)}, v_i) \right) \\
             &= \sum_{i=2}^n \left( \frac{x_{\pi(i),i}}{T} - \log \left( \sum_{k=1}^n e^{\frac{x_{ki}}{T}} \right) - \frac{x_{\pi'(i),i}}{T} + \log \left( \sum_{k=1}^n e^{\frac{x_{ki}}{T}} \right) \right) \\
             &= \sum_{i=2}^n  \left( \frac{x_{\pi(i),i}}{T} - \frac{x_{\pi'(i),i}}{T} \right) \\
             &= \frac{1}{T} \sum_{i=2}^n \left( x_{\pi(i),i} - x_{\pi'(i),i} \right)
\end{align}

Therefore, if $w_{X,T}(A) - w_{X,T}(A') \geq 0$, then $w_{X,T'}(A) - w_{X,T'}(A') \geq 0$ for all $T' \in \mathbb{R}_{>0}$.
Equivalently, if $w_{X,T}(A) \geq w_{X,T}(A')$, then $w_{X,T'}(A) \geq w_{X,T'}(A')$ for all $T' \in \mathbb{R}_{>0}$.
\end{proof}

We can now use the above result to prove our main claim:

\begin{theorem}
\label{thrm:main_claim}
A maximum arborescence $A$ of $(D, r)$ for a given $w_{X,T}$ is also a maximum arborescence of $(D, r)$ for any other $w_{X,T'}$ (with $T, T' \in \mathbb{R}_{>0}$).
\end{theorem}

\begin{proof}
Since $A$ is a maximum arborescence of $(D, r)$ and $w_{X,T}$, it holds by definition that $w_{X,T}(A) \geq w_{X,T}(A')$ for all arborescences $A'$ of $(D, r)$. From Theorem \ref{thrm:generalized_claim} it follows immediately that $w_{X,T'}(A) \geq w_{X,T'}(A')$ for all $A'$ also for any other $w_{X,T'}$.
\end{proof}

We have thus proven that applying temperature scaling to the logits of a graph-based dependency parser does not change the dependency tree returned by the overall system (assuming that maximum overall edge weight is used as the selection criterion).

\paragraph{Additional structural constraints.}
In practice, additional structural constraints are often imposed on dependency trees; most commonly, there may be only one edge emanating from the root node of the tree \cite{zmigrod-etal-2020-please}.
We note that our result is unaffected by these situations, as in this case, $A$ is simply a maximum arborescence chosen from a restricted set (i.e., all arboresences that also fulfill the additional structural criterion).
However, Theorem \ref{thrm:generalized_claim} still applies for all members of this restricted set, meaning that $A$ will still be the maximum arborescence from this set for any $w_{X,T'}$.

\section{Conclusion}
\label{sec:conclusion}

In this paper, we have proven that temperature scaling, a popular technique for post-hoc calibration of neural network classifiers, does not have any effect on the output of graph-based syntactic dependency parsers.
We conclude that more sophisticated methods -- in particular, those that may also change class predictions instead of only re-scaling output probabilities -- may be needed for tackling miscalibration in dependency parsers in ways that may lead to improved accuracy.
Investigating such methods could be a promising direction for future work.

\section*{Acknowledgements}
The author thanks Annemarie Friedrich and Sophie Henning for their helpful suggestions and feedback.

\bibliographystyle{coling}
\bibliography{coling2020}

\end{document}